%% file: main.tex
\documentclass[11pt,letterpaper]{article}

\usepackage[T1]{fontenc}
\usepackage{lmodern}
\usepackage[margin=1in]{geometry}
\usepackage{amsmath,amssymb,amsthm}
\usepackage{microtype}
\usepackage{booktabs}
\usepackage[round,authoryear]{natbib}
\usepackage{xcolor}
\usepackage{hyperref}
\hypersetup{
  colorlinks=true,
  linkcolor=blue,
  citecolor=blue,
  urlcolor=blue,
  pdftitle={Minimax Last-Iterate Convergence in Matrix Games with Observed Actions},
  pdfauthor={Yuheng Zhang}
}
\newtheorem{theorem}{Theorem}
\newtheorem{lemma}[theorem]{Lemma}

\newcounter{algorithm}
\newenvironment{algobox}[1]{%
  \begin{figure}[htbp]
  \refstepcounter{algorithm}
  \hrule\smallskip
  \noindent\textbf{Algorithm \thealgorithm. #1}\par\smallskip
  \hrule\smallskip\small
}{%
  \smallskip\hrule
  \end{figure}
}

\title{Minimax Last-Iterate Convergence in Matrix Games\\
with Observed Actions}

\author{Yuheng Zhang\\
University of Illinois Urbana-Champaign\\
\texttt{yuhengz2@illinois.edu}}
\date{}

\begin{document}

\maketitle

\begin{abstract}
We study last-iterate convergence in unknown two-player zero-sum matrix games with bandit payoff feedback and observed opponent actions. For games with $d$ actions per player, we develop an algorithm achieving a duality gap of $\widetilde{\mathcal{O}}(\sqrt{d/t})$ with high probability, simultaneously at every round $t$. This improves the dimension dependence of the best previously known guarantee by a factor of $d^{3/2}$. The rate matches a standard bandit lower bound, establishing minimax optimality in both the number of actions and the number of rounds, up to logarithmic factors. The algorithm is computationally efficient, requiring only $\mathcal{O}(d)$ time and memory per round. Our technical contribution is a joint design of adaptive averaging and
corrected exponential weights that absorbs estimation variance, together
with a potential argument that bounds phase durations.
\end{abstract}

\input{introduction}
\input{prelim}
\input{algorithm}
\input{theory}
\input{related_work}
\input{conclusion}

\subsection*{AI use statement}
We use GPT-6 Astra to polish the writing, assist with calculations
in the proofs, and check their correctness. We review all AI-assisted
content and take full responsibility for the final content of this paper.

\bibliography{ref}
\bibliographystyle{plainnat}

\clearpage
\appendix
\input{proofs}

\end{document}

%% file: introduction.tex
\section{Introduction}\label{sec:introduction}

We study learning a Nash equilibrium in unknown two-player zero-sum matrix
games with $d$ actions per player using feedback from sampled interactions.
Nash equilibrium is a central solution concept in learning in games;
in zero-sum games, equilibrium strategies guarantee each player the value
of the game against any opponent.
Classical no-regret algorithms guarantee
convergence of average strategies \citep{freund1999adaptive}, but the
strategies used at individual rounds can remain exploitable. When learning
and deployment occur together, players need guarantees on their current
strategies, not just on an average over past play. This motivates last-iterate guarantees,
which ensure that the strategies actually played approach equilibrium and
remain near it at all sufficiently late rounds.

We consider bandit payoff feedback with observed opponent actions:
after each round, both players
observe the sampled action pair and its common noisy payoff. Such feedback
is natural in self-play, where both actions can be recorded, and in
preference learning, where both responses in a comparison are available
\citep{munos2023nash}. Security games provide another example when the
attacked target and the defensive action are observable after each round
\citep{hait2026near}.
Observing the opponent's action provides information beyond the payoff
alone, while revealing only one payoff per round leaves exploration
necessary. We ask whether this additional information suffices to attain
statistically optimal last-iterate convergence.

Under bandit feedback without observing the opponent's action, prior work
obtains last-iterate duality gaps of
$\widetilde{\mathcal{O}}(\sqrt d\,t^{-1/8})$
\citep{cai2023uncoupled} and
$\widetilde{\mathcal{O}}(d^{1/5}t^{-1/5})$
\citep{cai2026average}, where $t$ is the number of rounds and
$\widetilde{\mathcal{O}}$ suppresses logarithmic factors.
\citet{fiegel2026optimal} improve the time dependence to $t^{-1/4}$,
up to logarithmic factors, with an anytime high-probability guarantee.
Their paper does not state the dependence on $d$ explicitly; making the
parameters in their proof explicit yields the bound
$\widetilde{\mathcal{O}}(d^2t^{-1/4})$.

\begin{table}[t]
  \centering
  \caption{Representative last-iterate guarantees for zero-sum matrix games
  with $d$ actions per player and bandit payoff feedback.
  All upper bounds hold with high probability uniformly over time.
  The last column gives sufficient rounds to attain and maintain duality
  gap at most $\varepsilon$; the row for the lower bound gives necessary rounds
  at fixed confidence. The gap lower bound assumes $t\ge d$.
  For \citet{fiegel2026optimal}, the dependence on $d$ is not stated in
  their paper, and the bounds reported here are derived from
  their proof (see details in Appendix~\ref{app:fiegel-dimension}).}
  \label{tab:comparison}
  \small
  \setlength{\tabcolsep}{5pt}
  \renewcommand{\arraystretch}{1.2}
  \begin{tabular}{@{}lccc@{}}
    \toprule
    Work & Opponent actions & Duality gap at round $t$
         & Rounds for gap $\varepsilon$ \\
    \midrule
    \citet{cai2023uncoupled} & Unobserved
      & $\widetilde{\mathcal{O}}(\sqrt d\,t^{-1/8})$
      & $\widetilde{\mathcal{O}}(d^4/\varepsilon^8)$ \\
    \citet{cai2026average} & Unobserved
      & $\widetilde{\mathcal{O}}(d^{1/5}t^{-1/5})$
      & $\widetilde{\mathcal{O}}(d/\varepsilon^5)$ \\
    \citet{fiegel2026optimal} & Unobserved
      & $\widetilde{\mathcal{O}}(d^2t^{-1/4})$
      & $\widetilde{\mathcal{O}}(d^8/\varepsilon^4)$ \\
    \midrule
    \citet{hait2026near} & Observed
      & $\widetilde{\mathcal{O}}(d^2/\sqrt t)$
      & $\widetilde{\mathcal{O}}(d^4/\varepsilon^2)$ \\
      Bandit lower bound & Observed
      & $\Omega(\sqrt{d/t})$
      & $\Omega(d/\varepsilon^2)$ \\
    \textbf{Our work} & \textbf{Observed}
      & $\boldsymbol{\widetilde{\mathcal{O}}(\sqrt{d/t})}$
      & $\boldsymbol{\widetilde{\mathcal{O}}(d/\varepsilon^2)}$ \\
    \bottomrule
  \end{tabular}
\end{table}

With observed opponent actions, the closest work, \citet{hait2026near},
achieves $\widetilde{\mathcal{O}}(d^2/\sqrt t)$ last-iterate convergence
with high probability by estimating the payoff matrix and periodically
solving a game with log-barrier regularization. This achieves the optimal
dependence on $t$, but leaves a gap in $d$ compared with the bandit lower
bound of order $\sqrt{d/t}$. Thus, a basic question remains:
\begin{quote}
\emph{Can last-iterate convergence with observed opponent actions attain
the minimax rate in both the number of actions and the number of rounds?}
\end{quote}

\paragraph{Our results.}
We answer this question affirmatively. To the best of our knowledge, we
provide the first high-probability last-iterate guarantee that is minimax
optimal in both $d$ and $t$, up to logarithmic factors, for bandit payoff
feedback with observed opponent actions. Specifically, for any confidence
parameter $\delta\in(0,1)$, our algorithm produces played strategies $(x_t,y_t)$
whose duality gap, the sum of both players' gains from unilateral
deviations, satisfies
\[
  \operatorname{Gap}(x_t,y_t)
  \le C\sqrt{\frac dt}
       \left[\log\!\left(\frac{dt}{\delta}\right)\right]^{3/2}
  \qquad\text{simultaneously for all }t\ge1
\]
with probability at least $1-\delta$, where $C$ is a universal constant
(Theorem~\ref{thm:upper}). The algorithm uses one payoff observation per
round and does not require the time horizon. It improves the dimension
dependence of \citet{hait2026near} by a factor of $d^{3/2}$.
Equivalently, the number of rounds sufficient to reach duality gap
$\varepsilon$ and maintain this accuracy thereafter decreases from
$\widetilde{\mathcal{O}}(d^4/\varepsilon^2)$ to
$\widetilde{\mathcal{O}}(d/\varepsilon^2)$.
Games with identical columns reduce to a $d$-armed bandit problem,
giving a worst-case lower bound of $\Omega(d/\varepsilon^2)$ observations
at fixed confidence \citep{mannor2004sample}.
This lower bound matches our guarantee up to logarithmic factors.
The algorithm also requires only $\mathcal{O}(d)$
time and memory per round, using explicit vector updates.
Table~\ref{tab:comparison} summarizes the comparison.

\paragraph{Technical ideas.}
We build on the idea of playing an average of auxiliary strategies
\citep{cai2026average}. Under bandit feedback, the auxiliary learners and
the played strategies use different distributions, so estimating the
auxiliary losses can incur large variance. Our main technical contribution
is a joint design of adaptive averaging and a correction to exponential
weights. The correction creates a negative term measuring the distribution
mismatch, while the averaging weights allow this term to absorb the
estimation variance with only a linear factor in $d$. Together with an
adaptation of implicit exploration \citep{neu2015explore}, this yields the
optimal dimension dependence with high probability.

Adaptive averaging can nevertheless make little progress in actual rounds
when the weights are small. Our second contribution is a potential analysis
that bounds the number of rounds needed for each constant-factor improvement
in accuracy. We combine this analysis with a phase scheme that retains the
preceding phase's output as a baseline, keeping intermediate play near
equilibrium while learning a more accurate strategy pair. This converts
the variance control into an anytime guarantee for every played pair.

%% file: prelim.tex
\section{Preliminaries}\label{sec:preliminaries}

\paragraph{Notation.}
For a positive integer $d$, let $[d]=\{1,\ldots,d\}$ and
$\Delta_d=\{x\in\mathbb R_{\ge 0}^d:\sum_{i=1}^d x_i=1\}$.
We write $x_i$ for the $i$th coordinate of a vector $x$, $e_i$ for the
$i$th standard basis vector, and $\mathbf 1$ for the all-ones vector.
All logarithms are natural, and $\widetilde{\mathcal{O}}$ hides logarithmic factors.

\paragraph{Problem formulation.}

We consider a two-player zero-sum game with a fixed unknown matrix
$A\in[-1,1]^{d\times d}$, where $d\ge 2$ is the number of actions available
to each player. The entry $A_{ij}$ is the expected loss of the row player
and the expected reward of the column player when they choose actions
$i$ and $j$, respectively. For mixed strategies $x,y\in\Delta_d$, the row
player minimizes $x^\top Ay$, while the column player maximizes it.

The players learn through repeated play with bandit payoff feedback and
observed opponent actions. Let $\mathcal F_{t-1}$ denote the common
observation history before round $t$. At round $t$, the players choose
mixed strategies $x_t,y_t\in\Delta_d$ based on this history.
Conditionally on $\mathcal F_{t-1}$, they independently sample
$I_t\sim x_t$ and $J_t\sim y_t$. Both players then observe the action pair
$(I_t,J_t)$ and a common payoff $R_t\in[-1,1]$ satisfying
\[
  \mathbb E[R_t\mid\mathcal F_{t-1},I_t,J_t]=A_{I_tJ_t}.
\]
Each round thus provides one payoff observation, interpreted as a loss
for the row player and a reward for the column player.

We measure the quality of a strategy pair $(x,y)$ by its duality gap,
\[
  \operatorname{Gap}(x,y)
  =\max_{y'\in\Delta_d}x^\top Ay'
     -\min_{x'\in\Delta_d}{x'}^\top Ay
  =\max_{j\in[d]}x^\top Ae_j-\min_{i\in[d]}e_i^\top Ay.
\]
This quantity is the sum of the two players' gains from unilateral
deviations and lies in $[0,2]$. A pair $(x,y)$ is a Nash equilibrium if
and only if $\operatorname{Gap}(x,y)=0$. More generally,
$\operatorname{Gap}(x,y)\le\varepsilon$ implies that $(x,y)$ is an
$\varepsilon$-Nash equilibrium: neither player can improve its expected
payoff by more than $\varepsilon$ through a unilateral deviation.

Our goal is last-iterate convergence, measured by the duality gap of
the strategies $(x_t,y_t)$ actually played at round $t$.
Specifically, given a confidence parameter $\delta\in(0,1)$, we seek an
algorithm that does not require the time horizon and satisfies
\[
  \mathbb P\!\left(
    \forall t\ge 1:\ 
    \operatorname{Gap}(x_t,y_t)\le B_d(t,\delta)
  \right)\ge 1-\delta,
\]
where $B_d(t,\delta)$ is a deterministic error bound that tends to zero
as $t\to\infty$.

%% file: algorithm.tex
\section{Algorithm}\label{sec:algorithm}

Our algorithm proceeds in phases, starting from the uniform strategy pair.
Each phase refines the pair returned by the preceding phase, aiming to reduce
its duality gap by a constant factor while controlling the gap at every
intermediate round.

Within each phase, we maintain an auxiliary strategy pair $(u_t,v_t)$ for
learning and a played pair $(x_t,y_t)$ for collecting feedback. The played
strategies are weighted averages of the auxiliary strategies and the
baseline pair used to initialize the phase. This follows the idea of turning an internal average into
the actual strategy, as in the A2L reduction of
\citet{cai2026average}. The difficulty is that the auxiliary players need
loss estimates against each other, but observations come from the played
pair. We address this mismatch by choosing the averaging weights adaptively
and adding a ratio correction to the exponential weights update.

\paragraph{Phase structure.}
A phase starts from a positive baseline $(p,q)$ with duality gap at most
$e$, where $e\in(0,2]$ is its accuracy parameter. Given a failure allowance
$\rho\in(0,1)$, it returns a pair with gap at most $e/2$, while controlling
the gap throughout the phase, with conditional probability at least
$1-\rho$. The returned pair serves as the baseline for the next phase,
whose accuracy parameter is halved.

Starting from $p=q=\mathbf 1/d$, we therefore run phases $k=0,1,\ldots$
with $e_k=2^{1-k}$ and $\rho_k=\delta/2^{k+1}$. The failure allowances sum
to $\delta$. The universal bound $\operatorname{Gap}(p,q)\le 2$ establishes
the initial accuracy, and each successful phase supplies the baseline
guarantee for the next one. These guarantees allow us to use the prescribed
accuracy levels without evaluating the unknown gap.
We now describe a single phase with parameters $(e,\rho)$, using
$t=1,2,\ldots$ to count rounds within that phase; this index restarts at
one at the beginning of each phase. The following paragraphs present the
averaging, loss estimation, and auxiliary update steps, followed by the
parameter choices and stopping rule.

\paragraph{Adaptive averaging.}
Initialize $x_1=u_1=p$, $y_1=v_1=q$, and $\tau_1=\tau_0$, where $\tau_0>0$
is the initial weight assigned to the baseline. At each round, choose a
weight $a_t\in(0,1]$ and set
\begin{equation}\label{eq:played-update}
  \tau_{t+1}=\tau_t+a_t,\qquad
  x_{t+1}=\frac{\tau_t x_t+a_tu_t}{\tau_{t+1}},\qquad
  y_{t+1}=\frac{\tau_t y_t+a_tv_t}{\tau_{t+1}}.
\end{equation}
Thus the baseline continues to contribute to the played pair while the
auxiliary learners collect enough information to improve it.
To choose the weight, define the ratios
\[
  r_{t,i}=\frac{u_{t,i}}{x_{t,i}},\qquad
  s_{t,j}=\frac{v_{t,j}}{y_{t,j}},\qquad
  S_{x,t}=\sum_{i=1}^d r_{t,i},\qquad
  S_{y,t}=\sum_{j=1}^d s_{t,j},
\]
and let
\begin{equation}\label{eq:clock}
  a_t=\min\left\{1,\frac{d}{S_{x,t}+S_{y,t}}\right\}.
\end{equation}
A large ratio means that an auxiliary strategy puts more mass on an action
than the corresponding sampling strategy does. Such actions require large
importance weights, so we reduce both the auxiliary update and its
contribution to the played average. The choice in \eqref{eq:clock} ensures
$a_t(S_{x,t}+S_{y,t})\le d$.
Both players use the weight $a_t$, which preserves the cancellation of their
payoffs when their weighted regrets are added.

\paragraph{Estimating the auxiliary losses.}
The auxiliary players learn with the nonnegative loss vectors
\[
  g_t=\frac{\mathbf 1+Av_t}{2},\qquad
  h_t=\frac{\mathbf 1-A^\top u_t}{2}.
\]
In round $t$, draw $I_t\sim x_t$ and $J_t\sim y_t$ independently and observe
$(I_t,J_t,R_t)$. A natural unbiased estimator of $g_{t,i}$ is
$\mathbf 1\{I_t=i\}(1+R_t)s_{t,J_t}/(2x_{t,i})$.
The factor $1/x_{t,i}$ corrects for sampling the row action, while
$s_{t,J_t}=v_{t,J_t}/y_{t,J_t}$ changes the opponent distribution from
$y_t$ to $v_t$. This estimator can have large variance when $x_{t,i}$
is small or the opponent ratio $s_{t,J_t}$ is large. To control this
variance, we use implicit exploration (IX), adapting the estimator of
\citet{neu2015explore}. For a learning rate $\eta>0$ fixed within the
phase, set $\zeta_t=\eta a_t$ and use
\begin{equation}\label{eq:estimates}
  \widehat g_{t,i}
  =\frac{\mathbf 1\{I_t=i\}(1+R_t)s_{t,J_t}}
         {2(x_{t,i}+\zeta_t s_{t,J_t})},\qquad
  \widehat h_{t,j}
  =\frac{\mathbf 1\{J_t=j\}(1-R_t)r_{t,I_t}}
         {2(y_{t,j}+\zeta_t r_{t,I_t})}.
\end{equation}
The added denominator terms introduce a downward bias but ensure that
each scaled estimate $\eta a_t\widehat g_{t,i}$ and
$\eta a_t\widehat h_{t,j}$ lies in $[0,1]$, even when sampling probabilities
are small. The IX correction scales with the opponent ratio and the
adaptive weight $a_t$, matching the sampling correction to the size of
the auxiliary update.

\paragraph{Exponential weights with ratio correction.}
We first compute $(x_{t+1},y_{t+1})$ from the current auxiliary pair
$(u_t,v_t)$ via \eqref{eq:played-update}. Actions are sampled from
$(x_t,y_t)$, and the loss estimates in \eqref{eq:estimates} use the current
strategies $(x_t,y_t,u_t,v_t)$. With these estimates, we update the auxiliary
strategies by
\begin{equation}\label{eq:auxiliary-update}
  u_{t+1,i}\propto
  u_{t,i}\exp(-\eta a_t\widehat g_{t,i})\frac{x_{t,i}}{x_{t+1,i}},
  \qquad
  v_{t+1,j}\propto
  v_{t,j}\exp(-\eta a_t\widehat h_{t,j})\frac{y_{t,j}}{y_{t+1,j}}.
\end{equation}
The correction factors $x_{t,i}/x_{t+1,i}$ and $y_{t,j}/y_{t+1,j}$
modify the usual exponential weights update. By \eqref{eq:played-update},
\[
  \frac{x_{t,i}}{x_{t+1,i}}
  =\frac{1+a_t/\tau_t}{1+(a_t/\tau_t)r_{t,i}}.
\]
This identity shows that the correction downweights coordinates with large
ratios between the auxiliary and played probabilities. The same holds for
the column player.
In the regret bound, these corrections produce a negative term measuring
the discrepancy between the auxiliary and played distributions. Together
with the adaptive weights in \eqref{eq:clock}, this term absorbs the
estimation variance of both players. For any fixed comparison action,
the logarithms of the correction factors telescope across rounds,
leaving only terms involving the initial and final played probabilities.

\paragraph{Phase parameters and termination.}
The initial weight $\tau_0$ controls how strongly the played strategies
retain the baseline, while the learning rate $\eta$ controls the auxiliary
updates and implicit exploration through $\zeta_t=\eta a_t$.
To choose them, define
\[
  R_0=\log\frac{1}{\min_i p_i}+\log\frac{1}{\min_j q_j},
  \qquad L=\log\frac{2d+2}{\rho},
  \qquad H=R_0+\log 5+2L.
\]
Here $R_0$ accounts for the initial comparison cost associated with small
baseline probabilities, $L$ accounts for the failure allowance $\rho$,
and $H$ collects these logarithmic terms. We set
\begin{equation}\label{eq:phase-parameters}
  \tau_0=\max\left\{6d,
    \left\lceil\frac{16384dH^2}{e^2}\right\rceil\right\},
  \qquad \eta=\frac{16H}{e\tau_0}.
\end{equation}
The baseline weight keeps the played strategies close enough to the
baseline while the auxiliary learners improve their average.
The corresponding learning rate allows the ratio correction to control
estimation error while keeping the auxiliary regret small enough to
reduce the gap. Here $16384$ is a sufficiently large constant.

We end the phase at the first round $N$ with $\tau_{N+1}\ge4\tau_0$, or
equivalently $\sum_{t=1}^N a_t\ge3\tau_0$, and return
$(x_{N+1},y_{N+1})$. The auxiliary strategies then contribute at least
three times the baseline weight. With the above parameter choices,
this is enough to reduce the gap below $e/2$, while keeping it at most
$5e/4$ throughout the phase, with conditional probability at least $1-\rho$.

Algorithm~\ref{alg:main} gives the complete procedure. Both players can
compute the played and auxiliary strategies from their common observations
and the input $(d,\delta)$. The updates require no additional communication
or shared randomness. All coordinates remain positive, and the algorithm stores
only a constant number of length-$d$ vectors. Computing the weights,
sampling actions, and performing the normalized updates takes
$\mathcal{O}(d)$ operations per round, with $\mathcal{O}(d)$ memory.

\begin{algobox}{Phased exponential weights with ratio correction}\label{alg:main}
\begin{tabbing}
\qquad\=\qquad\=\qquad\=\kill
\textbf{Input:} Number of actions $d\ge2$ and confidence $\delta\in(0,1)$.\\
Initialize $p=q=\mathbf 1/d$.\\
\textbf{for} $k=0,1,2,\ldots$ \textbf{do}\\
\>Set $e=2^{1-k}$ and $\rho=\delta/2^{k+1}$; compute $\tau_0,\eta$ by \eqref{eq:phase-parameters}.\\
\>Initialize $x=u=p$, $y=v=q$, and $\tau=\tau_0$.\\
\>\textbf{while} $\tau<4\tau_0$ \textbf{do}\\
\>\>Compute $r_i=u_i/x_i$, $s_j=v_j/y_j$, and $a$ by \eqref{eq:clock}; set $\zeta=\eta a$.\\
\>\>Form $\tau^+=\tau+a$, $x^+=(\tau x+au)/\tau^+$, and $y^+=(\tau y+av)/\tau^+$.\\
\>\>Independently draw $I\sim x$ and $J\sim y$; observe $(I,J,R)$.\\
\>\>Compute $\widehat g,\widehat h$ by \eqref{eq:estimates} using $x,y,u,v$ and $(I,J,R)$.\\
\>\>Compute $u^+,v^+$ by \eqref{eq:auxiliary-update} using $x,y,x^+,y^+$ and normalize.\\
\>\>Set $(x,y,u,v,\tau)\leftarrow(x^+,y^+,u^+,v^+,\tau^+)$.\\
\>\textbf{end while}\\
\>Set $(p,q)\leftarrow(x,y)$.\\
\textbf{end for}
\end{tabbing}
\end{algobox}

\paragraph{Comparison with prior algorithms.}
Compared with \citet{hait2026near}, our algorithm controls estimation
variance directly through adaptive averaging and ratio correction.
Their method uses log-barrier regularization of an estimated payoff
matrix to control estimation error. We instead estimate the auxiliary
loss vectors directly: the weights $a_t$ limit the aggregate
importance ratios across both players, and the ratio correction
downweights coordinates with large ratios between auxiliary and played
probabilities. By jointly controlling these ratios and the resulting
variance, our method improves the dependence on the number of actions.

This design also gives simpler updates. Whereas \citet{hait2026near}
solve a regularized game at each epoch boundary and keep the strategies
fixed within the epoch, our algorithm performs explicit vector updates
each round. It requires only $\mathcal{O}(d)$ time and memory per round,
without constructing a payoff matrix estimate or solving a regularized
game.

%% file: theory.tex
\section{Theoretical Guarantees and Analysis}\label{sec:theory}

\subsection{Last-iterate Guarantees}\label{sec:guarantee}

We first state the convergence guarantee for Algorithm~\ref{alg:main}.
In the following guarantee, $t$ counts the total number of interaction rounds
across all phases.

\begin{theorem}[High-probability last-iterate convergence]\label{thm:upper}
Fix $d\ge2$, $A\in[-1,1]^{d\times d}$, and $\delta\in(0,1)$.
Under the feedback model in Section~\ref{sec:preliminaries}, with
probability at least $1-\delta$, the played strategies of
Algorithm~\ref{alg:main} satisfy, simultaneously for all $t\ge1$,
\[
  \operatorname{Gap}(x_t,y_t)
  \le C\sqrt{\frac dt}
      \left[\log\!\left(\frac{dt}{\delta}\right)\right]^{3/2},
\]
where $C>0$ is a universal constant.
\end{theorem}

Theorem~\ref{thm:upper} gives a high-probability last-iterate guarantee
for the mixed strategies actually used to sample actions. Although the
played strategies are constructed by averaging auxiliary strategies, the
bound applies directly to $(x_t,y_t)$ at round $t$. Consequently, for any
target accuracy $\varepsilon\in(0,1]$, after
$\widetilde{\mathcal{O}}(d/\varepsilon^2)$ rounds, every subsequent played
pair is an $\varepsilon$-Nash equilibrium.

Under the same feedback model, \citet[Theorem~4.1]{hait2026near} establish
a high-probability bound of
$\widetilde{\mathcal{O}}(d^2/\sqrt t)$ simultaneously for all
rounds. Theorem~\ref{thm:upper} preserves the $t^{-1/2}$ dependence and
improves the polynomial dependence on the number of actions from $d^2$
to $\sqrt d$, a factor of $d^{3/2}$. In terms of observations needed to reach accuracy
$\varepsilon$, the dimension dependence improves from
$\widetilde{\mathcal{O}}(d^4/\varepsilon^2)$ to
$\widetilde{\mathcal{O}}(d/\varepsilon^2)$.
This improvement comes from matching the adaptive averaging weights to
the ratio correction. The weights control the aggregate importance ratios
of both players, while the correction absorbs the estimation variance
caused by differences between auxiliary and played strategies. By
controlling the variance of the auxiliary loss estimates directly, we
avoid the extra dimension factors incurred when transferring entrywise
payoff estimates between successive strategy pairs.

The dependence on $d$ and $t$ is minimax optimal up to logarithmic
factors, by the standard bandit pure-exploration lower bound
\citep[Theorem~1]{mannor2004sample}. With identical columns, the gap is
the row strategy's excess loss in a $d$-armed bandit, and column actions
provide no additional information. By Markov's inequality, sampling from an
$\varepsilon$-optimal mixture returns a $c\varepsilon$-optimal arm with
probability at least $1-1/c$ for any constant $c>1$. Thus, achieving gap
at most $\varepsilon$ with sufficiently high fixed confidence requires
$\Omega(d/\varepsilon^2)$ observations in the worst case. Equivalently,
for $t\ge d$, the worst-case gap cannot improve on the
$\Omega(\sqrt{d/t})$ scale. Theorem~\ref{thm:upper} matches this dependence
while providing a guarantee simultaneously for every round without
knowing the horizon.
To the best of our knowledge, this is the first work to achieve
high-probability last-iterate convergence that is minimax optimal in both
$d$ and $t$, up to logarithmic factors, under this feedback model.

\subsection{Theoretical Analysis}\label{sec:analysis}

The proof follows the algorithm's phase structure. Within each phase,
the adaptive weights allow the ratio correction to absorb the estimation
variance, yielding a joint weighted regret bound for the auxiliary
learners. This bound controls the gap of their weighted average.
Mixing that average with the baseline then controls every intermediate
played pair and produces a more accurate baseline for the next phase.
This progress is measured in accumulated averaging weight, so we also
bound the number of interaction rounds needed to complete each phase.
Combining the duration bound with the geometrically decreasing accuracy
parameters yields the global last-iterate guarantee.

We develop the single-phase analysis in three steps: bounding the joint
weighted regret, converting this bound into a gap guarantee for the
played strategies, and controlling the phase duration. We then combine
the phases to prove Theorem~\ref{thm:upper}.
The full proofs are in Appendix~\ref{app:analysis}.

We begin by bounding the auxiliary learners' regret within a phase
initialized at a positive baseline $(p,q)$, with accuracy parameter $e$
and failure allowance $\rho$. Here and in the next two steps, $t$ counts
rounds within the phase and probabilities are conditional on the history
at its start.

We study the sum of the learners' regrets with the weights $a_t$,
since this joint regret controls the duality gap of their weighted
average. The following lemma bounds it for every prefix up to the phase's
stopping round, denoted by $N$.

\begin{lemma}[Joint regret bound]\label{lem:joint-regret}
With probability at least $1-\rho$, for every prefix
$1\le n\le N$ and every $i,j\in[d]$,
\begin{equation}\label{eq:joint-regret}
  \sum_{t=1}^n a_t
  \bigl(
    \langle u_t,g_t\rangle-g_{t,i}
    +\langle v_t,h_t\rangle-h_{t,j}
  \bigr)
  \le\frac{2H}{\eta}.
\end{equation}
\end{lemma}

This bound applies to the true losses $g_t,h_t$, although the learners
update using IX estimates sampled from the played pair. Its right side
depends only on the phase parameters and contains no accumulated variance
term. This allows the auxiliary average to become more accurate as its
total weight grows, even when the auxiliary and played strategies differ
substantially.

To establish this bound, we first identify the variance cost that must
be controlled. Define
\[
  K_{x,t}=\sum_i\frac{u_{t,i}^2}{x_{t,i}},
  \qquad K_{y,t}=\sum_j\frac{v_{t,j}^2}{y_{t,j}}.
\]
Each $K$ measures the discrepancy between an auxiliary strategy and its
sampling distribution, and equals one when they agree. The IX
concentration argument gives a variance cost proportional to
$\eta\sum_t a_t^2(S_{x,t}K_{y,t}+S_{y,t}K_{x,t})$.
To absorb this cost, we seek a negative contribution to regret with
the same dependence on $K_{x,t}$ and $K_{y,t}$.

The ratio correction provides such a contribution. We derive it first
for the row player. Writing $\kappa_t=a_t/\tau_t$,
\eqref{eq:played-update} gives
$x_{t,i}/x_{t+1,i}=(1+\kappa_t)/(1+\kappa_t r_{t,i})$.
The common factor $1+\kappa_t$ cancels in the auxiliary update, so this
update is exponential weights with modified costs
$\eta a_t\widehat g_{t,i}+c_{t,i}$, where
$c_{t,i}=\log(1+\kappa_t r_{t,i})$.
The exponential weights regret inequality includes the learner's
correction cost $\langle u_t,c_t\rangle$ minus the comparison action's
cost $c_{t,i}$. Moving these costs to the right side introduces
\[
  \frac1\eta\sum_{t=1}^n
  \bigl(c_{t,i}-\langle u_t,c_t\rangle\bigr)
\]
in the bound on the row player's weighted regret for the estimated losses.
The comparison action's contribution telescopes:
$\sum_{t=1}^n c_{t,i}
=\log(\tau_{n+1}x_{n+1,i}/(\tau_0p_i))$,
leaving only a logarithmic cost.
The learner's contribution has a negative sign, and
$\log(1+z)\ge z-z^2/2$ gives
\[
  -\langle u_t,c_t\rangle
  \le-\kappa_tK_{x,t}
     +\frac{\kappa_t^2}{2}
        \sum_{\ell=1}^d u_{t,\ell}r_{t,\ell}^2.
\]
Thus the row player's regret bound contains
$-\eta^{-1}\sum_t\kappa_tK_{x,t}$, together with quadratic remainders.
Applying the same argument to the column player gives a regret bound
containing $-\eta^{-1}\sum_t\kappa_tK_{y,t}$.

We now add the row and column players' regret bounds for the estimated losses.
To obtain regret for the true losses, the IX concentration argument
controls the estimation error together with the second-order loss terms
from exponential weights. This introduces the variance cost identified
above. After using $\tau_0\ge6d$ to control the quadratic costs of the
ratio correction, the left side of \eqref{eq:joint-regret} is bounded by
\[
  \frac{2H}{\eta}
  +3\eta\sum_{t=1}^n a_t^2(S_{x,t}K_{y,t}+S_{y,t}K_{x,t})
  -\frac{3}{4\eta}\sum_{t=1}^n\kappa_t(K_{x,t}+K_{y,t}).
\]
It remains to show that the negative correction term in this bound
absorbs the positive variance cost. The adaptive weights satisfy
$a_t(S_{x,t}+S_{y,t})\le d$, giving
\[
  a_t^2(S_{x,t}K_{y,t}+S_{y,t}K_{x,t})
  \le d a_t(K_{x,t}+K_{y,t})
  =d\tau_t\kappa_t(K_{x,t}+K_{y,t}).
\]
This inequality limits the contribution of the potentially large sums
$S_{x,t}$ and $S_{y,t}$ to a single factor $d$. The remaining dependence
on $\kappa_t(K_{x,t}+K_{y,t})$ matches that of the negative correction term.
Since the parameter choice ensures $\eta^2d\tau_t\le1/4$ throughout the
phase, the variance cost is absorbed and only $2H/\eta$ remains.
Appendix~\ref{app:regret} gives the full argument.

We next transfer this regret guarantee to the played strategies by
tracking the weight contributed by the auxiliary learners.
Let $A_n=\sum_{t=1}^n a_t$ denote their cumulative averaging weight after
$n$ rounds, with $A_0=0$. Together with the baseline weight $\tau_0$,
this gives total weight $\tau_{n+1}=\tau_0+A_n$, and the phase stops
when $A_N\ge3\tau_0$.

Early in the phase, the auxiliary average may still be inaccurate, so
we also use the accuracy of the baseline. The following lemma shows that
mixing with the baseline controls every intermediate gap and produces
a more accurate pair at the stopping round.

\begin{lemma}[Duality gap bound]\label{lem:phase}
Suppose $\operatorname{Gap}(p,q)\le e$. On the event of
Lemma~\ref{lem:joint-regret}, every $0\le n\le N$ satisfies
\begin{equation}\label{eq:phase-gap}
  \operatorname{Gap}(x_{n+1},y_{n+1})
  \le\frac{\tau_0e+4H/\eta}{\tau_0+A_n}
  =\frac{5e\tau_0}{4(\tau_0+A_n)}.
\end{equation}
In particular, the gap is at most $5e/4$ throughout the phase, and the
returned pair has gap at most $5e/16<e/2$.
\end{lemma}

The two conclusions serve distinct roles. The intermediate bound preserves
accuracy during the rounds used to learn a better baseline, while the
endpoint contraction allows the next phase to start with accuracy parameter
$e/2$. Together, these bounds control the gap of every played pair across
successive phases, which is essential for the last-iterate guarantee.

To see why the lemma holds, let
$\bar u_n=A_n^{-1}\sum_{t=1}^n a_tu_t$ and
$\bar v_n=A_n^{-1}\sum_{t=1}^n a_tv_t$ for $n\ge1$.
Using the weights $a_t$ for both players makes their bilinear payoffs cancel:
\[
  \frac{A_n}{2}\operatorname{Gap}(\bar u_n,\bar v_n)
  =
  \max_{i,j}\sum_{t=1}^n a_t
  \bigl(
    \langle u_t,g_t\rangle-g_{t,i}
    +\langle v_t,h_t\rangle-h_{t,j}
  \bigr).
\]
Lemma~\ref{lem:joint-regret} therefore gives
$A_n\operatorname{Gap}(\bar u_n,\bar v_n)\le4H/\eta$.
The played pair mixes this average with the baseline using weights
$A_n$ and $\tau_0$. Convexity of the gap and the parameter choice
$4H/\eta=e\tau_0/4$ then yield \eqref{eq:phase-gap};
see Appendix~\ref{app:phase}.

Lemma~\ref{lem:phase} measures progress in accumulated weight $A_n$.
To obtain a convergence rate in interaction rounds, we must bound how long
the phase takes to reach $A_N\ge3\tau_0$, even when individual weights
are small. The next lemma provides this bound and controls the smallest
coordinates of the returned pair, which determine the next phase's
parameters.

\begin{samepage}
\begin{lemma}[Phase length and coordinate bounds]\label{lem:duration}
For every positive baseline $(p,q)$, the following bounds hold.
\begin{enumerate}
  \renewcommand{\labelenumi}{(\roman{enumi})}
  \item \textnormal{Phase length.} The number of interaction rounds satisfies
  \[
    N=\mathcal{O}\bigl(\tau_0(1+R_0)\bigr)
     =\mathcal{O}\!\left(\frac{dH^3}{e^2}\right).
  \]
  \item \textnormal{Coordinate lower bounds.} For every $0\le n\le N$
  and every $i,j\in[d]$, the played strategies satisfy
  \[
    x_{n+1,i}\ge\frac{p_i}{5},
    \qquad y_{n+1,j}\ge\frac{q_j}{5}.
  \]
\end{enumerate}
\end{lemma}
\end{samepage}

The duration bound shows that adaptive weighting costs only the logarithmic
factor $1+R_0$ beyond the baseline weight $\tau_0$. The coordinate bounds
also limit the increase of $R_0$ from one phase to the next to $2\log5$.
These two properties keep the dependence on baseline probabilities
logarithmic across phases, preserving the polynomial dependence
$d/e^2$ in the number of rounds needed to improve accuracy.

The proof tracks the unnormalized played coordinates
$M_{t,i}=\tau_tx_{t,i}$. Their logarithmic increments satisfy
\[
  M_{t+1,i}=M_{t,i}+a_tu_{t,i},
  \qquad
  \log\frac{M_{t+1,i}}{M_{t,i}}=\log(1+\kappa_t r_{t,i}).
\]
Summing these increments telescopes. Applying the same argument to the
column player and using $1\le a_t+(a_t/d)(S_{x,t}+S_{y,t})$ gives
\[
  n\le A_n+\tau_{n+1}\left(1+\frac d{\tau_0}\right)
       \left(R_0+2\log\frac{\tau_{n+1}}{\tau_0}\right).
\]
The stopping rule bounds $A_N$ and $\tau_{N+1}$ by constant multiples of
$\tau_0$, yielding the duration bound. The coordinate bounds follow from
monotonicity of $M_{t,i}$ and the same bound on total weight.
Appendix~\ref{app:duration} gives the details.

Finally, let $N_k$ be the duration of phase $k$, and let $R_{0,k}$,
$H_k$, and $\tau_{0,k}$ denote the corresponding phase parameters.
The initial baseline coordinates are $1/d$, and Lemma~\ref{lem:duration} ensures that each
new baseline coordinate is at least one fifth of its previous value.
After $k$ phases, they are therefore at least $1/(d5^k)$, giving
$R_{0,k}\le2\log d+2k\log5$.
Writing $B_k=\log(d/\delta)+k+1$, we obtain
$H_k=\mathcal{O}(B_k)$ and $N_k=\mathcal{O}(dB_k^3/e_k^2)$.
If global round $T$ lies in phase $k$, the geometric accuracy schedule
yields
\[
  T\le\sum_{\ell=0}^kN_\ell
    =\mathcal{O}\!\left(\frac{dB_k^3}{e_k^2}\right).
\]
Moreover, $N_\ell\ge3\tau_{0,\ell}\ge4^\ell$, so
$k=\mathcal{O}(\log T+1)$.
The phase failure allowances sum to $\delta$, and
Lemma~\ref{lem:phase} gives gap at most $5e_k/4$ in every successful phase.
Substituting the bound on $T$ and
$B_k=\mathcal{O}(\log(dT/\delta))$ proves Theorem~\ref{thm:upper}.

%% file: related_work.tex
\section{Related Work}\label{sec:related}

\paragraph{Last-iterate convergence.}
With exact gradient feedback, last-iterate convergence is well understood
in several classes of games: optimistic gradient methods achieve
instance-dependent linear rates in matrix games \citep{wei2020linear},
and accelerated methods attain an $\mathcal{O}(1/t)$ rate in smooth
monotone games \citep{cai2023doubly}. With bandit payoff feedback, learning
must also account for estimation error and continued exploration.
For unknown matrix games without observed opponent actions,
\citet{cai2023uncoupled,cai2026average,fiegel2026optimal} establish
high-probability last-iterate guarantees with rates scaling as
$t^{-1/8}$, $t^{-1/5}$, and $t^{-1/4}$, respectively, up to
dimension-dependent and logarithmic factors.
Other works study last-iterate convergence under different
convergence criteria or additional assumptions.
\citet{fiegel2026harder} obtain an
$\widetilde{\mathcal{O}}((d/t)^{1/4})$ rate for the $L^2$ norm of the
duality gap, which controls its second moment at each round rather than
providing a single high-probability event covering all rounds.
\citet{ito2025instance} prove last-iterate convergence in expectation
under a unique pure-strategy Nash equilibrium assumption.
Payoff-based best-response dynamics also admit finite-sample guarantees
in stochastic and polymatrix games
\citep{chen2024decentralized,faizal2024finite}.
A separate line studies zeroth-order feedback in continuous action
spaces \citep{dong2026uncoupled,maiti2026efficient}, where players
observe evaluations of the payoff function at their chosen continuous
actions, rather than sampled matrix entries as in our setting.

\paragraph{Self-play learning in games.}
Self-play is a standard approach to learning strategies in games,
with algorithms based on regret minimization for equilibrium computation
and reinforcement learning for sequential decision making
\citep{freund1999adaptive,zinkevich2007regret,bai2020provable,liu2021sharp,zhang2026beyond}.
It has enabled strong empirical performance in board games and poker
\citep{silver2018general,brown2019superhuman}.
More recently, self-play has been applied to language model alignment
with general preferences, where potentially nontransitive comparisons
motivate learning a Nash policy
\citep{munos2023nash,zhang2025iterative,zhang2026improving,wu2025self}. In self-play, the sampled actions of both players can be recorded,
making observed opponent actions a natural feedback model.
We study this setting in unknown two-player zero-sum matrix games
with bandit payoff feedback.
Under this feedback model, \citet{o2021matrix} bound
cumulative regret relative to the game value against arbitrary
opponents, whereas we study last-iterate convergence to the equilibrium. The closest work to ours, \citet{hait2026near}, establishes a
high-probability last-iterate rate of
$\widetilde{\mathcal{O}}(d^2/\sqrt{t})$.
We improve this rate to $\widetilde{\mathcal{O}}(\sqrt{d/t})$,
matching the minimax lower bound up to logarithmic factors in both
the number of actions and the number of rounds.

%% file: conclusion.tex
\section{Conclusion}\label{sec:conclusion}

We study last-iterate convergence in unknown two-player zero-sum
matrix games with bandit payoff feedback and observed opponent actions.
We develop an algorithm that combines adaptive averaging with
corrected exponential weights to control estimation variance.
For games with $d$ actions per player, our algorithm achieves
a duality gap of $\widetilde{\mathcal{O}}(\sqrt{d/t})$ with high
probability uniformly over all rounds $t$.
This improves the $\widetilde{\mathcal{O}}(d^2/\sqrt{t})$ bound of
\citet{hait2026near} by a factor of $d^{3/2}$ and matches the minimax
lower bound in both $d$ and $t$ up to logarithmic factors.

%% file: proofs.tex
\section{Proofs of the theoretical guarantees}\label{app:analysis}

We use the notation of Section~\ref{sec:analysis}. Throughout
Appendices~\ref{app:regret}--\ref{app:duration}, $t$ is the local round
index and all probabilistic statements are conditional on the history
at the start of the phase. The baseline and phase parameters are fixed
under this conditioning. If the phase starts after global round $\sigma$,
write $\mathcal H_t=\mathcal F_{\sigma+t}$ and
$\mathbb E_t[\cdot]=\mathbb E[\cdot\mid\mathcal H_{t-1}]$.

\subsection{Joint weighted regret}\label{app:regret}

\subsubsection{Conditional concentration of the loss estimates}
\label{app:concentration}

The estimates in \eqref{eq:estimates} are biased downward.
We need a bound for each comparison action and a separate bound for the
learner's aggregate estimation error.
To retain the variance term that will later be absorbed, define
\begin{equation}\label{eq:variance-budgets}
  \beta_t^x
     =\sum_{i,j}\frac{u_{t,i}v_{t,j}^2}
                       {x_{t,i}y_{t,j}+\zeta_t v_{t,j}},
  \qquad
  \beta_t^y
     =\sum_{i,j}\frac{u_{t,i}^2v_{t,j}}
                       {x_{t,i}y_{t,j}+\zeta_t u_{t,i}}.
\end{equation}
These quantities are used only in the analysis and are not computed by
the algorithm.
We also introduce the corrected learner losses
\[
  W_t^x=a_t\langle u_t,\widehat g_t\rangle
                  -\eta a_t^2\langle u_t,\widehat g_t^2\rangle,
  \qquad
  W_t^y=a_t\langle v_t,\widehat h_t\rangle
                  -\eta a_t^2\langle v_t,\widehat h_t^2\rangle.
\]
The subtracted quadratic terms coincide with the second-order cost in
the entropy update, allowing us to analyze the estimation and update costs together.

To pass from estimated regret to true regret, we need to control
overestimation of each comparison action's loss and underestimation of
the learner's loss. The following lemma provides both bounds by adapting
the exponential moment argument for implicit exploration
\citep{neu2015explore} to the two-player estimates and predictable weights.

\begin{lemma}[Simultaneous estimation bounds]\label{lem:concentration}
With probability at least $1-\rho$, the following inequalities
hold for every prefix $n\le N$ and every $i,j\in[d]$:
\begin{equation}\label{eq:comparator-concentration}
  \sum_{t=1}^n a_t(\widehat g_{t,i}-g_{t,i})\le\frac L\eta,
  \qquad
  \sum_{t=1}^n a_t(\widehat h_{t,j}-h_{t,j})\le\frac L\eta,
\end{equation}
and
\begin{equation}\label{eq:learner-concentration}
  \begin{aligned}
  \sum_{t=1}^n\bigl(a_t\langle u_t,g_t\rangle-W_t^x\bigr)
       &\le3\eta\sum_{t=1}^n a_t^2\beta_t^x+\frac L\eta,\\
  \sum_{t=1}^n\bigl(a_t\langle v_t,h_t\rangle-W_t^y\bigr)
       &\le3\eta\sum_{t=1}^n a_t^2\beta_t^y+\frac L\eta.
  \end{aligned}
\end{equation}
\end{lemma}

The two bounds account for the different roles of the estimated losses
in regret. Comparison actions incur only the confidence cost $L/\eta$,
while the learner bounds retain the variance terms $\beta_t^x,\beta_t^y$.
Keeping these terms explicit allows the ratio correction to absorb them
in the proof of Lemma~\ref{lem:joint-regret}.

\begin{proof}
We prove the statements for the row player; the column proof replaces
$(1+R_t)/2$ by $(1-R_t)/2$.
Suppress the time index, write $\ell=(1+R)/2$, and set
\[
  \mu_{ij}=\frac{1+A_{ij}}2,\qquad
  D_{ij}=x_i y_j+\zeta v_j.
\]
Then $\widehat g_i=\mathbf 1\{I=i\}\ell v_J/D_{iJ}$ and
$g_i=\sum_jv_j\mu_{ij}$.
Conditioning on the current action pair gives
\[
  g_i-\mathbb E_t\widehat g_i
   =\zeta\sum_j\frac{v_j^2\mu_{ij}}{D_{ij}}\ge0.
\]
Since $0\le\ell^2\le\ell\le1$, the second moment satisfies
\[
  \mathbb E_t\langle u,\widehat g^2\rangle
  \le\sum_{i,j}\frac{u_i x_i y_j v_j^2}{D_{ij}^2}
  \le\sum_{i,j}\frac{u_i v_j^2}{D_{ij}}
  =\beta^x.
\]
Together with the bias identity above and $D_{ij}\ge\zeta v_j$, this yields
\begin{equation}\label{eq:conditional-moments}
  \langle u,g-\mathbb E_t\widehat g\rangle\le\zeta\beta^x,\qquad
  \mathbb E_t\langle u,\widehat g^2\rangle\le\beta^x,\qquad
  0\le\zeta\widehat g_i\le1.
\end{equation}
For a comparison coordinate, we need the stronger moment inequality
\[
  \begin{aligned}
  \mathbb E_t[\widehat g_i+\zeta\widehat g_i^2]
  &\le
  \sum_jv_j\mu_{ij}
    \left(\frac{x_i y_j}{D_{ij}}
          +\frac{\zeta x_i y_jv_j}{D_{ij}^2}\right)
  \le g_i.
  \end{aligned}
\]
Indeed, for $b,c\ge0$ with $b+c>0$,
$b/(b+c)+bc/(b+c)^2=1-c^2/(b+c)^2\le1$.
Using $e^z\le1+z+z^2$ for $0\le z\le1$ therefore yields
\[
  \mathbb E_t\exp\{\zeta(\widehat g_i-g_i)\}
  \le e^{-\zeta g_i}(1+\zeta g_i)\le1.
\]
Because $\zeta_t=\eta a_t$, the process
$\exp(\eta\sum_{t=1}^n a_t(\widehat g_{t,i}-g_{t,i}))$
is a nonnegative supermartingale starting at one, up to the end of the
phase.

For the learner, \eqref{eq:conditional-moments} implies
\[
  W^x=a\sum_i u_i\widehat g_i(1-\zeta\widehat g_i)\ge0,
  \qquad
  \mathbb E_tW^x\le a\le1.
\]
Furthermore, Jensen's inequality and
$W^x\le a\langle u,\widehat g\rangle$ give
\[
  \mathbb E_t[(W^x)^2]\le a^2\beta^x,\qquad
  a\langle u,g\rangle-\mathbb E_tW^x\le2\eta a^2\beta^x.
\]
Put $Y_t=\mathbb E_tW_t^x-W_t^x$. Then
$\mathbb E_tY_t=0$, $Y_t\le1$, and $\mathbb E_tY_t^2\le a_t^2\beta_t^x$.
The parameter choice \eqref{eq:phase-parameters} ensures
$\eta\le e/(1024dH)<1$.
Using $e^z\le1+z+z^2$ for all $z\le1$ gives
\[
  \mathbb E_t\exp\{\eta Y_t-\eta^2a_t^2\beta_t^x\}\le1.
\]
Consequently
$\exp(\eta\sum_{t=1}^nY_t-\eta^2\sum_{t=1}^n a_t^2\beta_t^x)$
is also a nonnegative supermartingale.
Its maximal inequality contributes
$\eta\sum_t a_t^2\beta_t^x+L/\eta$ to the learner's deviation.
Adding the preceding conditional bias accounts for the coefficient three
in \eqref{eq:learner-concentration}.

We now combine these bounds into a single event that holds uniformly
over all prefixes of the phase. Extend each exponential process above
by keeping it constant after the phase ends. Since whether round $t$
is executed is determined by $\mathcal H_{t-1}$, this extension preserves
the one-step conditional inequalities. Each extended process is
therefore a nonnegative supermartingale starting at one.

By Ville's inequality, each process remains below $e^L$ at every prefix
with probability at least $1-e^{-L}$. For each comparison coordinate, taking logarithms gives
\eqref{eq:comparator-concentration}. For each learner, taking logarithms
bounds the cumulative deviation $\sum_{t=1}^n Y_t$; adding the conditional
bias bound above then gives \eqref{eq:learner-concentration}.
A union bound over the $2d$ comparison processes and the two learner
processes shows that all these inequalities hold simultaneously with
probability at least $1-(2d+2)e^{-L}=1-\rho$.
\end{proof}

\subsubsection{Regret with ratio correction}

We next derive a deterministic bound for the exponential weights update
with ratio correction. Recall $K_{x,t},K_{y,t}$ and $\kappa_t$ from
Section~\ref{sec:analysis}, and define
\[
  J_{x,t}=\sum_i u_{t,i}r_{t,i}^2,
  \qquad J_{y,t}=\sum_j v_{t,j}s_{t,j}^2.
\]
These quantities bound the quadratic cost of the correction, whereas
$K_{x,t}$ and $K_{y,t}$ enter with a negative sign.
The following lemma isolates this negative contribution in the regret
bound and tracks the accompanying quadratic cost.

\begin{lemma}[Regret with ratio correction]\label{lem:damping}
For every prefix $n\le N$ and every row action $i$,
\begin{equation}\label{eq:damping-bound}
  \eta\sum_{t=1}^n(W_t^x-a_t\widehat g_{t,i})
  \le 2\log\frac1{p_i}+\log5
      -\sum_{t=1}^n\kappa_tK_{x,t}
      +\frac32\sum_{t=1}^n\kappa_t^2J_{x,t}.
\end{equation}
The same inequality holds for the column player with
$(W^x,\widehat g,p,K_x,J_x)$ replaced by
$(W^y,\widehat h,q,K_y,J_y)$.
\end{lemma}

The negative term in \eqref{eq:damping-bound} is the contribution that
compensates for the learner's estimation error in
Lemma~\ref{lem:concentration}. The adaptive weights keep the quadratic
term small enough to retain this benefit, while the comparison cost
depends only logarithmically on the baseline probability.

\begin{proof}
The common factor $1+\kappa_t$ in $x_{t,i}/x_{t+1,i}$ disappears after
normalization in \eqref{eq:auxiliary-update}. Thus the row update is
exponential weights with nonnegative costs
\[
  b_{t,i}=\zeta_t\widehat g_{t,i}+\log(1+\kappa_t r_{t,i}),
  \qquad
  Z_t=\sum_i u_{t,i}e^{-b_{t,i}}.
\]
Using $e^{-b}\le1-b+b^2/2$ for $b\ge0$ and $\log z\le z-1$,
we have
\[
  \log Z_t
  \le-\langle u_t,b_t\rangle+\frac12\langle u_t,b_t^2\rangle
  \le-\eta W_t^x-\kappa_tK_{x,t}
                       +\frac32\kappa_t^2J_{x,t}.
\]
The last inequality uses
$(s+\log(1+z))^2/2\le s^2+z^2$ and
$\log(1+z)\ge z-z^2/2$ for $s,z\ge0$.
The exact normalized update now gives
\[
  \eta(W_t^x-a_t\widehat g_{t,i})
  \le\log\frac{u_{t+1,i}}{u_{t,i}}
       +\log(1+\kappa_t r_{t,i})
       -\kappa_tK_{x,t}+\frac32\kappa_t^2J_{x,t}.
\]
The second logarithm also telescopes, since
\[
  1+\kappa_t r_{t,i}
   =\frac{\tau_{t+1}x_{t+1,i}}{\tau_t x_{t,i}},
  \qquad
  \sum_{t=1}^n\log(1+\kappa_t r_{t,i})
   =\log\frac{\tau_{n+1}x_{n+1,i}}{\tau_0p_i}.
\]
For $n\le N$, the stopping rule and $a_t\le1$ imply
$\tau_{n+1}\le5\tau_0$.
Summing the previous inequality and bounding
$u_{n+1,i},x_{n+1,i}\le1$ proves \eqref{eq:damping-bound}.
\end{proof}

\subsubsection{Proof of the joint regret bound}

\begin{proof}[Proof of Lemma~\ref{lem:joint-regret}]
Work on the event of Lemma~\ref{lem:concentration}.
Dropping the positive implicit exploration terms in the denominators of
\eqref{eq:variance-budgets} gives
$\beta_t^x\le S_{x,t}K_{y,t}$ and
$\beta_t^y\le S_{y,t}K_{x,t}$.
Using $a_t(S_{x,t}+S_{y,t})\le d$, we obtain
\begin{equation}\label{eq:variance-matching}
  a_t^2(\beta_t^x+\beta_t^y)
  \le a_t^2(S_{x,t}K_{y,t}+S_{y,t}K_{x,t})
  \le d a_t(K_{x,t}+K_{y,t}).
\end{equation}
Likewise, $a_t r_{t,i}\le d$ and $a_t s_{t,j}\le d$ imply
\[
  \kappa_t^2(J_{x,t}+J_{y,t})
     \le\frac d{\tau_0}\kappa_t(K_{x,t}+K_{y,t}).
\]
Combining \eqref{eq:comparator-concentration},
\eqref{eq:learner-concentration}, and \eqref{eq:damping-bound} for both
players, and applying \eqref{eq:variance-matching}, bounds the left side of
\eqref{eq:joint-regret} by
\[
  \frac{2R_0+2\log5+4L}{\eta}
  +\frac1\eta\sum_{t=1}^n\kappa_t(K_{x,t}+K_{y,t})
      \left(3\eta^2d\tau_t+\frac{3d}{2\tau_0}-1\right).
\]
Here the $4L/\eta$ term comprises one comparison and one learner
concentration cost for each player.
For every executed round, $\tau_t\le5\tau_0$, and the two lower bounds on
$\tau_0$ in \eqref{eq:phase-parameters} ensure
\[
  3\eta^2d\tau_t+\frac{3d}{2\tau_0}
  \le\frac{3840}{16384}+\frac14
  =\frac{31}{64}<1.
\]
The sum is therefore nonpositive. Dropping it and using
$H=R_0+\log5+2L$ gives \eqref{eq:joint-regret} simultaneously for every
$1\le n\le N$ and every $i,j\in[d]$ on the event of
Lemma~\ref{lem:concentration}. This event has probability at least
$1-\rho$, completing the proof.
\end{proof}

\subsection{Gap within a phase}\label{app:phase}

\begin{proof}[Proof of Lemma~\ref{lem:phase}]
For $n\ge1$, let
\[
  \overline u_n=\frac1{A_n}\sum_{t=1}^n a_tu_t,
  \qquad
  \overline v_n=\frac1{A_n}\sum_{t=1}^n a_tv_t.
\]
Since $g_t=(\mathbf 1+Av_t)/2$ and $h_t=(\mathbf 1-A^\top u_t)/2$,
the common term $u_t^\top Av_t$ cancels:
\[
  \begin{aligned}
  &\max_{i,j}\sum_{t=1}^n a_t
    \bigl(\langle u_t,g_t\rangle-g_{t,i}
                +\langle v_t,h_t\rangle-h_{t,j}\bigr)\\
  &\hspace{2em}=
    \frac12\max_{i,j}\sum_{t=1}^n
        a_t(u_t^\top Ae_j-e_i^\top Av_t)
    =\frac{A_n}{2}\operatorname{Gap}(\overline u_n,\overline v_n).
  \end{aligned}
\]
This identity holds for every realized sequence of the predictable
weights. The factor $1/2$ accounts for the shift of the payoff to
nonnegative losses.
Lemma~\ref{lem:joint-regret} consequently gives
$A_n\operatorname{Gap}(\overline u_n,\overline v_n)\le4H/\eta$.
On the other hand, \eqref{eq:played-update} implies
\[
  (x_{n+1},y_{n+1})
  =\frac{\tau_0(p,q)+A_n(\overline u_n,\overline v_n)}
         {\tau_0+A_n}.
\]
Using joint convexity of the duality gap and the baseline assumption,
\[
  \operatorname{Gap}(x_{n+1},y_{n+1})
  \le\frac{\tau_0e+4H/\eta}{\tau_0+A_n}
  =\frac{5e\tau_0}{4(\tau_0+A_n)},
\]
where the last equality uses \eqref{eq:phase-parameters}.
For $n=0$, \eqref{eq:phase-gap} follows directly from
$\operatorname{Gap}(p,q)\le e$.
At the endpoint $A_N\ge3\tau_0$, which proves the contraction.
\end{proof}

\subsection{Phase duration}\label{app:duration}

We prove the following explicit bound, which implies
Lemma~\ref{lem:duration}:
\begin{equation}\label{eq:duration}
  N\le
  3\tau_0+1+
  (4\tau_0+1)\left(1+\frac d{\tau_0}\right)
                    (R_0+2\log5).
\end{equation}

\begin{proof}[Proof of Lemma~\ref{lem:duration}]
For the row player, set $M_{t,i}=\tau_tx_{t,i}$.
The averaging update \eqref{eq:played-update} gives
\[
  M_{t+1,i}=M_{t,i}+a_tu_{t,i},
  \qquad
  z_{t,i}:=\frac{M_{t+1,i}-M_{t,i}}{M_{t,i}}
           =\kappa_t r_{t,i}\le\frac d{\tau_0}.
\]
The last inequality follows from \eqref{eq:clock} and
$\tau_t\ge\tau_0$.
For $0\le z\le d/\tau_0$,
$z\le(1+d/\tau_0)\log(1+z)$.
Since $\tau_t\le\tau_{n+1}$, summing over rounds and coordinates yields
\[
  \begin{aligned}
  \sum_{t=1}^n a_tS_{x,t}
  &=\sum_{t=1}^n\tau_t\sum_i z_{t,i}\\
  &\le\tau_{n+1}\left(1+\frac d{\tau_0}\right)
       \sum_i\log\frac{\tau_{n+1}x_{n+1,i}}{\tau_0p_i}.
  \end{aligned}
\]
The logarithms telescope because $1+z_{t,i}=M_{t+1,i}/M_{t,i}$.
Each coordinate of $x_{n+1}$ is at most one, so the final coordinate sum
is at most
$d[\log(\tau_{n+1}/\tau_0)+\log(1/\min_i p_i)]$.
Applying the same argument to the column player and using the pointwise
inequality
\[
  1\le a_t+\frac{a_t}{d}(S_{x,t}+S_{y,t})
\]
gives the prefix bound
\begin{equation}\label{eq:physical-prefix}
  n\le A_n+
  \tau_{n+1}\left(1+\frac d{\tau_0}\right)
     \left(R_0+2\log\frac{\tau_{n+1}}{\tau_0}\right).
\end{equation}

We first show that the phase ends after finitely many rounds. Otherwise,
the stopping condition would never be met, so $A_n<3\tau_0$ and
$\tau_{n+1}<4\tau_0$ for every $n$. The prefix bound
\eqref{eq:physical-prefix} would then imply
\[
  n\le 3\tau_0+
  4\tau_0\left(1+\frac d{\tau_0}\right)(R_0+2\log4)
  \qquad\text{for every }n\ge1,
\]
which is impossible because the right side is independent of $n$.

We can therefore apply \eqref{eq:physical-prefix} at the stopping round
$N$. By definition, $A_{N-1}<3\tau_0$, and $a_N\le1$ gives
$A_N=A_{N-1}+a_N<3\tau_0+1$. Thus
$\tau_{N+1}=\tau_0+A_N<4\tau_0+1\le5\tau_0$.
Substituting these bounds into \eqref{eq:physical-prefix} with $n=N$
proves \eqref{eq:duration}. Since $\tau_0\ge6d$, this gives
$N=\mathcal{O}(\tau_0(1+R_0))$.
The parameter choice \eqref{eq:phase-parameters} further gives
$\tau_0=\mathcal{O}(dH^2/e^2)$, while $1+R_0=\mathcal{O}(H)$.
Hence $N=\mathcal{O}(dH^3/e^2)$, proving part~(i).

For part~(ii), the update
$M_{t+1,i}=M_{t,i}+a_tu_{t,i}$ shows that each $M_{t,i}$ is
nondecreasing. For every $0\le n\le N$, we therefore have
$M_{n+1,i}\ge\tau_0p_i$ and
$\tau_{n+1}\le\tau_{N+1}\le5\tau_0$, so
\[
  x_{n+1,i}=\frac{M_{n+1,i}}{\tau_{n+1}}
  \ge\frac{\tau_0p_i}{5\tau_0}=\frac{p_i}{5}.
\]
Applying the same argument to $\tau_ty_{t,j}$ gives
$y_{n+1,j}\ge q_j/5$ for every $j\in[d]$. These are the coordinate
bounds in part~(ii), completing the proof.
\end{proof}

\subsection{The global guarantee}\label{app:global}

\begin{proof}[Proof of Theorem~\ref{thm:upper}]
We restore the global round index $t$. Lemma~\ref{lem:duration}
ensures that every phase ends. Conditional on any history at the start of a phase,
Lemma~\ref{lem:concentration} has failure probability at most $\rho_k$.
Taking expectations and summing $\rho_k=\delta/2^{k+1}$ over phases
gives an event $\mathcal E$ of probability at least $1-\delta$ on which
all phase concentration bounds hold. On this event,
Lemma~\ref{lem:joint-regret} applies in every phase.
The initial baseline has gap at most $e_0=2$, and
Lemma~\ref{lem:phase} inductively supplies a baseline of gap at most
$e_k$ for phase $k$. Every pair played in that phase then has gap at
most $5e_k/4$.

Let $N_k$ denote the duration of phase $k$ and attach a phase subscript
to its parameters. The coordinate bounds in Lemma~\ref{lem:duration}
hold on every history. Starting from the uniform baseline, they give
$p_i,q_j\ge1/(d5^k)$ in phase $k$, and hence
\[
  R_{0,k}\le2\log d+2k\log5,
  \qquad
  L_k=\log\frac{(2d+2)2^{k+1}}{\delta}.
\]
With $B_k=\log(d/\delta)+k+1$, the parameter choices in
\eqref{eq:phase-parameters} therefore imply
\[
  H_k=\mathcal{O}(B_k),\qquad
  \tau_{0,k}=\mathcal{O}\!\left(\frac{dB_k^2}{e_k^2}\right).
\]
Lemma~\ref{lem:duration} now gives a universal constant $C_1$ such that
\begin{equation}\label{eq:phase-length-global}
  N_k\le C_1\frac{dB_k^3}{e_k^2}.
\end{equation}

If global round $t$ lies in phase $k$, then
$t\le\sum_{j=0}^kN_j$. Since
$e_j^{-2}=4^{j-k}e_k^{-2}$ and $B_j\le B_k$, summing
\eqref{eq:phase-length-global} yields
\begin{equation}\label{eq:time-to-phase}
  t\le\frac{4C_1dB_k^3}{3e_k^2}.
\end{equation}
To bound $B_k$ in terms of $t$, note that $a_t\le1$ implies
$N_j\ge3\tau_{0,j}\ge4^j$. The last inequality follows from
\eqref{eq:phase-parameters}, $H_j\ge1$, and $e_j^{-2}=4^{j-1}$.
For $k\ge1$, phase $k-1$ is complete before round $t$, so
$t\ge N_{k-1}\ge4^{k-1}$. Thus $k\le1+\log_4t$ for every phase
containing $t$, including $k=0$, and
$B_k=\mathcal{O}(\log(dt/\delta))$.

On $\mathcal E$, combining the gap bound $5e_k/4$ with
\eqref{eq:time-to-phase} gives
\[
  \operatorname{Gap}(x_t,y_t)
  \le\frac54\sqrt{\frac{4C_1dB_k^3}{3t}}
  \le C\sqrt{\frac dt}
     \left[\log\!\left(\frac{dt}{\delta}\right)\right]^{3/2}
\]
for a universal constant $C$. This holds for every $t\ge1$ on
$\mathcal E$, proving the theorem.
\end{proof}

\section{Dimension dependence of the log-barrier bound}
\label{app:fiegel-dimension}

We derive the dimension dependence reported for
\citet{fiegel2026optimal} in Table~\ref{tab:comparison}.
Their Theorem~5.2 bounds the duality gap by
$2K\tau\log((t+T_0)/\delta)(t+T_0)^{-1/4}$, where $K$ is the total
number of actions of both players. Their Assumption~5.1 leaves constants
depending on $K$ implicit. We give an explicit parameter choice below.

For our setting, $K=2d$. To distinguish their parameters from ours,
write $\eta_{\mathrm F},\tau_{\mathrm F},T_{\mathrm F}$ for their
$\eta,\tau,T_0$. Set $\delta_*=\min\{\delta,\exp(-2)\}$ and let $M$
be a sufficiently large universal constant. Define
\[
  \Lambda=\log\!\left(\frac{MK}{\delta_*}\right),\qquad
  \eta_{\mathrm F}=\frac{1}{M\Lambda^2},\qquad
  \tau_{\mathrm F}=M^2K\Lambda^2,\qquad
  T_{\mathrm F}=\left\lceil\tau_{\mathrm F}^4\right\rceil.
\]
Their learning rate and regularization schedules are then
$\eta_{\mathrm F}(t+T_{\mathrm F})^{-3/4}$ and
$\tau_{\mathrm F}\log((t+T_{\mathrm F})/\delta_*)
(t+T_{\mathrm F})^{-1/4}$, respectively.
In particular, $\eta_{\mathrm F}\tau_{\mathrm F}=MK$, and
\[
  \eta_{\mathrm F}\tau_{\mathrm F}
  \le\frac{T_{\mathrm F}}{(\log T_{\mathrm F})^4},
  \qquad
  T_{\mathrm F}
  \le[\log(1/\delta_*)]^2\tau_{\mathrm F}^4,
\]
as required by their Assumption~5.1.

It remains to check the dimension factors hidden in the smallness
conditions of their proof. Let $\theta_0$ denote their initial
regularization strength:
\[
  \theta_0=\tau_{\mathrm F}\log(T_{\mathrm F}/\delta_*)
                     T_{\mathrm F}^{-1/4}.
\]
Since $\log(T_{\mathrm F}/\delta_*)=\mathcal{O}(\Lambda)$, we have
$\theta_0=\mathcal{O}(\Lambda)$. Their Lemma~6.2, Lemma~B.1, and
Proposition~C.5 use the quantities
\[
  L_{\mathrm F}=\sqrt K,\qquad
  \sigma_{\mathrm F}=2+\theta_0\sqrt K,
  \qquad
  \rho_{\mathrm F}
  =\sigma_{\mathrm F}\eta_{\mathrm F}
      (4L_{\mathrm F}+2\theta_0)+\frac{\theta_0\sqrt K}{4}.
\]
The proposed parameters give
\[
  \sigma_{\mathrm F}=\mathcal{O}(\sqrt K\Lambda),\qquad
  \rho_{\mathrm F}
     =\mathcal{O}\!\left(\frac{K}{M\Lambda}+\sqrt K\Lambda\right),
\]
with universal constants independent of $M,K,\delta$.

The term caused by the changing regularization requires
$K/(\eta_{\mathrm F}\tau_{\mathrm F})$ to be sufficiently small.
Here this ratio equals $1/M$. The remaining drift and fluctuation
coefficients are controlled by
\[
\max\left\{
  \frac{\eta_{\mathrm F}^2L_{\mathrm F}\sigma_{\mathrm F}^3}
       {\tau_{\mathrm F}^2},\,
  \frac{\eta_{\mathrm F}^2\sigma_{\mathrm F}^4}
       {\tau_{\mathrm F}^2},\,
  \frac{\rho_{\mathrm F}\eta_{\mathrm F}\sigma_{\mathrm F}^2}
       {\tau_{\mathrm F}^2},\,
  \frac{\rho_{\mathrm F}^2}{\tau_{\mathrm F}^2},\,
  \frac{\eta_{\mathrm F}\sigma_{\mathrm F}^2}{\tau_{\mathrm F}},\,
  \frac{\rho_{\mathrm F}}{\tau_{\mathrm F}}
\right\}
=\mathcal{O}(M^{-2}).
\]
Denote these six coefficients, in the displayed order, by
$c_1,\ldots,c_6$. For example, for $M\ge1024$,
$\theta_0\le13\Lambda$, $\sigma_{\mathrm F}\le14\sqrt K\Lambda$,
and $\rho_{\mathrm F}/\tau_{\mathrm F}\le5/M^2$, which give
$\max_i c_i\le5/M^2$. The same choice ensures
\[
  \sigma_{\mathrm F}\eta_{\mathrm F}T_{\mathrm F}^{-3/4}
  \le\frac{14}{M^7K^{5/2}\Lambda^7}\le\frac1{32}.
\]
This verifies their Assumption~C.1 and the weaker condition
$\sigma_{\mathrm F}\eta_{\mathrm F}T_{\mathrm F}^{-3/4}
\le1/\sqrt{12}$ of their Lemma~B.3.

We next make the residual recursion explicit to check that its
first-order terms are absorbed as well. Write $w_t$ for their combined
iterate, $p_t$ for the regularized residual defined in their Appendix~B,
and $D_t=\|p_t\|_{*,w_t}^2$ for its squared dual local norm. For
$t\ge0$, set
\[
  s=t+T_{\mathrm F},\qquad h_t=\log(s/\delta_*),\qquad
  a_t=\eta_{\mathrm F}s^{-3/4},\qquad
  \theta_t=\tau_{\mathrm F}h_ts^{-1/4}.
\]
Let $v_t=F_{\theta_{t+1}}(w_{t+1})-F_{\theta_t}(w_t)$, where
$F_\theta$ is their regularized operator, and define
\[
  Z_t=\|p_t\|_{*,w_{t+1}}^2-D_t
             +2\langle v_t,p_t\rangle_{*,w_t}.
\]
The squared-norm expansion in their Appendix~D, with the full factor
$2$ in its cross term, gives
\[
  D_{t+1}-D_t
  \le Z_t+32\rho_{\mathrm F}\sigma_{\mathrm F}^2a_ts^{-3/4}
             +\rho_{\mathrm F}^2s^{-3/2}.
\]
Here we used their Proposition~C.5 and the pathwise norm-variation
bound of Proposition~C.3. For the conditional mean of the norm
variation, their Proposition~C.2 and Lemma~C.4 give the bound
\[
  \mathbb E_t[\|p_t\|_{*,w_{t+1}}^2-D_t]
  \le(2a_t\sqrt{D_t}+68\sigma_{\mathrm F}^2a_t^2)D_t
  \le2a_tD_t^{3/2}+68\sigma_{\mathrm F}^4a_t^2,
\]
where $\mathbb E_t$ conditions on the history before this update and
$D_t\le\sigma_{\mathrm F}^2$. Combining this with their Lemma~D.3
and Proposition~C.5 yields
\begin{align*}
  \mathbb E_t Z_t
  &\le-2a_t\theta_{t+1}D_t+2a_tD_t^{3/2}
       +\frac{\theta_t}{2s}\sqrt{KD_t}
       +(128L_{\mathrm F}\sigma_{\mathrm F}^3
                         +68\sigma_{\mathrm F}^4)a_t^2,\\
  |Z_t-\mathbb E_t Z_t|
  &\le(16\eta_{\mathrm F}\sigma_{\mathrm F}^2+8\rho_{\mathrm F})
             s^{-3/4}\sqrt{D_t}.
\end{align*}

To leave enough contraction to absorb the $2a_tD_t^{3/2}$ term, use
the normalization
\[
  U_t=\frac{16D_t}{\tau_{\mathrm F}^2h_t}.
\]
The stopping boundary $U_t\le h_t/\sqrt s$ in their Lemma~6.5 now
corresponds to $D_t\le\theta_t^2/16$. Before this boundary is crossed,
\[
  2a_tD_t^{3/2}\le\frac12a_t\theta_tD_t,
  \qquad
  \frac{\theta_t}{2s}\sqrt{KD_t}
  \le\frac14a_t\theta_tD_t+\frac{K\theta_t}{4a_ts^2}.
\]
Their Lemma~B.2 gives
$\theta_{t+1}\ge(1-1/(4s))\theta_t$. Since $s\ge2$, the first-order
terms above sum to at most
\[
  \left(-2+\frac1{2s}+\frac12+\frac14\right)a_t\theta_tD_t
  \le-a_t\theta_tD_t\le-\frac{D_t}{s},
\]
where the last inequality uses $\eta_{\mathrm F}\tau_{\mathrm F}h_t\ge1$.
As $h_{t+1}\ge h_t$ and $D_{t+1}\ge0$, the normalized recursion is
\[
  U_{t+1}\le(1-1/s)U_t+b_{t+1}+W_{t+1},\qquad
  W_{t+1}=\frac{16(Z_t-\mathbb E_tZ_t)}{\tau_{\mathrm F}^2h_t},
\]
with $\mathbb E_tW_{t+1}=0$ and remainder bounds
\begin{align*}
  b_{t+1}
  &\le\left[\frac4M+16(128c_1+68c_2+32c_3+c_4)\right]s^{-3/2}
   \le(s+1)^{-3/2},\\
  |W_{t+1}|
  &\le(64c_5+32c_6)s^{-3/4}\sqrt{U_t},\qquad
  W_{t+1}^2\le\tfrac12(s+1)^{-3/2}U_t.
\end{align*}
These inequalities hold for $M\ge1024$ by $\max_i c_i\le5/M^2$.
At the uniform initialization, the log-barrier gradient lies in the
normal cone, so $D_0\le4$ and
\[
  U_0\le\frac{64}{\tau_{\mathrm F}^2\log(T_{\mathrm F}/\delta_*)}
       \le\frac{\log T_{\mathrm F}}{\sqrt{T_{\mathrm F}}}.
\]
Their Lemma~6.5 applies to this recursive upper bound as well: its
exponential-supermartingale proof uses only the monotonicity of the
exponential at that step. Thus, with probability at least
$1-\delta_*$, simultaneously for all $t\ge0$,
$D_t\le\theta_t^2/16$. Their Lemma~6.4 then gives duality gap at
most $2K\theta_t$ in their $[0,1]$ loss normalization.

For our payoff range $[-1,1]$, run their algorithm on the losses
$(1+R_t)/2$ and let $(x_t,y_t)$ denote its strategies in this appendix.
The duality gap in our normalization is twice the gap for these losses.
Consequently, with probability at least $1-\delta$, simultaneously for
every $t\ge1$,
\[
  \operatorname{Gap}(x_t,y_t)
  \le 4K\tau_{\mathrm F}
       \log\!\left(\frac{t+T_{\mathrm F}}{\delta_*}\right)
       (t+T_{\mathrm F})^{-1/4}
  =\widetilde{\mathcal{O}}(d^2t^{-1/4}).
\]
Indeed, $\tau_{\mathrm F}=\widetilde{\mathcal{O}}(K)$ and
$T_{\mathrm F}=\widetilde{\mathcal{O}}(K^4)$, so the logarithm introduces
no additional polynomial dependence on $K$. Solving this bound for a
target gap $\varepsilon$ gives the round bound
$\widetilde{\mathcal{O}}(d^8/\varepsilon^4)$ stated in
Table~\ref{tab:comparison}.